\documentclass[letterpaper, 10 pt, conference]{ieeeconf}  % Comment this line out if you need a4paper

\IEEEoverridecommandlockouts           

\usepackage{cite}
\usepackage{amsmath,amssymb,amsfonts}
\usepackage{graphicx}
\usepackage{textcomp}
\def\BibTeX{{\rm B\kern-.05em{\sc i\kern-.025em b}\kern-.08em
    T\kern-.1667em\lower.7ex\hbox{E}\kern-.125emX}}
\usepackage[utf8]{inputenc}
\usepackage{mathrsfs}
\usepackage{xcolor}
\usepackage{mathtools}
\usepackage{enumerate}
\usepackage{booktabs}
\usepackage{romannum}
\usepackage{float}
\usepackage{algorithm}
\usepackage{algpseudocode}
\usepackage{multirow}
\usepackage{fancyvrb}
\usepackage{url}

\usepackage{setspace}
\renewcommand{\baselinestretch}{0.96}

\usepackage{amsthm}
\newtheorem{theorem}{Theorem}
\newtheorem{lm}{Lemma}

\newtheorem{definition}{Definition}

\newtheorem{rem}{Remark}
\newtheorem{prop}{Proposition}
\usepackage[normalem]{ulem}

\makeatletter
\renewcommand*\env@matrix[1][\arraystretch]{%
  \edef\arraystretch{#1}%
  \hskip -\arraycolsep
  \let\@ifnextchar\new@ifnextchar
  \array{*\c@MaxMatrixCols c}}

\DeclareUnicodeCharacter{2212}{-}
\title{\LARGE \bf
Game-Theory-Assisted Reinforcement Learning for Border Defense: Early Termination based on Analytical Solutions
}

\author{Goutam Das$^{1}$, Michael Dorothy$^{2}$, Kyle Volle$^{3}$, and Daigo Shishika$^{4}$% 
\thanks{We gratefully acknowledge the support of  ARL grant ARL DCIST CRA
W911NF-17-2-0181 and AFRL grant FA8651-23-1-0012.
The views expressed in this paper are those of the authors and do not reflect the official policy or position of the United States Government, Department of Defense, or its components.}% <-this % stops a space
\thanks{$^{1}$Goutam Das, Postdoctoral Researcher, School of Aeronautics and Astronautics, Purdue University, 610 Purdue Mall, West Lafayette, IN 47907, USA
        {\tt\small das347@purdue.edu}}%
\thanks{$^{2}$Michael Dorothy. Army Research Directorate, DEVCOM Army Research Laboratory, APG, MD 21005, USA
        {\tt\small michael.r.dorothy.civ@army.mil}}%
\thanks{$^{3}$Kyle Volle, is with the Munitions Directorate, Air Force Research Laboratory, Eglin AFB, FL 32542, USA {\tt\small kyle@us.af.mil.}}%
\thanks{$^{4}$Daigo Shishika, Assistant Professor, Department of Mechanical Engineering, George Mason University,
        4400 University Dr, Fairfax, VA 22030, USA
        {\tt\small dshishik@gmu.edu}}%
}

\begin{document}
\maketitle
\thispagestyle{empty}
\pagestyle{empty}

%%%%%%%%%%%%%%%%%%%%%%%%%%%%%%%%%%%%%%%%%%%%%%%%%%%%%%%%%%%%%%%%%%%%%%%%%%%%%%%%

\begin{abstract}
Game theory provides the gold standard for analyzing adversarial engagements, offering strong optimality guarantees. However, these guarantees often become brittle when assumptions such as perfect information are violated. Reinforcement learning (RL), by contrast, is adaptive but can be sample-inefficient in large, complex domains. This paper introduces a hybrid approach that leverages game-theoretic insights to improve RL training efficiency. We study a border defense game with limited perceptual range, where defender performance depends on both search and pursuit strategies, making classical differential game solutions inapplicable. Our method employs the Apollonius Circle (AC) to compute equilibrium in the post-detection phase, enabling early termination of RL episodes without learning pursuit dynamics. This allows RL to concentrate on learning search strategies while guaranteeing optimal continuation after detection. Across single- and multi-defender settings, this early termination method yields 10–20\% higher rewards, faster convergence, and more efficient search trajectories. 
% For heterogeneous teams with specialized roles, it also enhances coordination. 
Extensive experiments validate these findings and demonstrate the overall effectiveness of our approach.
\end{abstract}

\section{Introduction}
\noindent
Differential games model strategic interactions between competing agents in adversarial scenarios, initially formalized by Isaacs in his seminal work \cite{Isaacs1965}. 
% PEGs have applications ranging from missile guidance and air combat to search-and-rescue operations and robotic surveillance \cite{weintraub2020introduction}. 
Border defense games~\cite{VonMoll2020BD}--also known as target defense~\cite{Garcia2017Active,fu2020guarding} or perimeter defense games~\cite{shishika2019perimeter}--represent a specialized variant where a team of defenders must protect a target region or boundary from infiltrating attackers \cite{weintraub2020introduction,Shishika2020}. 
% In these scenarios, the intruder (attacker) team tries to score by reaching the target region while a team of defenders try to minimize the score by intercepting those intruders \cite{fu2020guarding}.

Analytical approaches yield exact solutions to these problems under simplifying assumptions on the dynamics, environments, and information structure.
% have yielded exact solutions for specific configurations. 
Isaacs established the relevance of Apollonius Circles (AC) for defining \emph{dominance regions} in pursuit-evasion under simple motion, which was later formalized in \cite{dorothy2024one}. 
Recent decomposition methods enable polynomial-time assignment algorithms for perimeter defense \cite{Shishika2018, Garcia2021Multiple, bajaj2022competitive,yan2019task}, with solutions for 
% single-defender \cite{garcia2019pursuit}, 
cooperative teams \cite{milutinovic2025cooperative, fuchs2010cooperative}, and heterogeneous attackers \cite{lee2024guarding, das2024heterogeneous}.

However, analytical methods face fundamental limitations under partial information \cite{shishika2021partial, maity2024cooperative, bopardikar2008discrete}.
% They require complete state knowledge, become computationally intractable with increasing agents, and often decompose team games into set of 1v1 or 2v1 subgames, potentially missing emergent cooperative behaviors \cite{das2024heterogeneous}. 
The stochastic nature of the search problem under uncertainty about the attacker's state makes it challenging to directly apply differential game techniques developed for full-state feedback settings. Moreover, standard decomposition into 1v1 or 2v1 subgames may overlook emergent cooperative behaviors.

Multi-agent reinforcement learning (MARL) methods such as MADDPG and MAPPO 
% \cite{lowe2017multi, yu2022surprising} 
address partial observability and discover cooperative strategies without complete system models. 
% Yet existing RL approaches learn end-to-end, from sensing through capture, without leveraging available analytical solutions \cite{wang2020cooperative}. This results in sample inefficiency, requiring extensive computational resources to relearn known optimal capture strategies.
 Yet existing RL approaches learn both search and pursuit behaviors end-to-end, without leveraging available analytical solutions \cite{wang2020cooperative}. This results in sample inefficiency, requiring extensive computational resources to relearn known optimal capture strategies.

% This paper presents a hybrid framework combining differential game theory with MARL for border defense under partial observability. We employ RL for the stochastic search phase where analytical solutions are intractable, while using Apollonius Circle optimization to compute Nash equilibrium payoffs for early termination upon sensing. This eliminates learning the deterministic pursuit phase, focusing computational resources on search coordination.
% Our contributions are threefold: 
This paper presents a hybrid framework combining differential game theory with MARL for border defense under partial observability.
% , using
% Apollonius Circle optimization 
% the Apollonius Circle to analytically compute 
% Nash equilibrium payoffs for early termination upon sensing. 
Upon detecting the attacker, we use the Apollonius Circle to analytically compute the Nash equilibrium payoff of the resulting pursuit game, replacing simulation of the pursuit phase with this closed-form reward during RL training.

Our contributions are threefold:
\begin{itemize}
    \item We present a novel hybrid framework that combines analytical game-theory (GT) solutions in the MARL training loop to improve learning efficiency,
    \item We empirically demonstrate that this approach achieves 
    % better spatial configurations, 
    significantly higher rewards and faster convergence compared to end-to-end learning.
    \item We show that our method is effective across different team sizes, enabling the learning of robust search and coordination strategies.
\end{itemize}

% \textbf{Paper Organization:} Section II formulates the two-phase border defense problem. Section III presents the Apollonius Circle method for Nash equilibrium computation. Section IV details GT-assisted early termination integration with MARL. Section V provides experimental evaluation using MAPPO. Section VI concludes with future directions.
\textbf{Paper Organization:} Section II formulates the two-phase border defense problem. Section III presents the Apollonius Circle method for Nash equilibrium computation. Section IV details the MARL formulation with GT-assisted early termination. Section V provides experimental evaluation using MAPPO. Section VI concludes with future directions.

\section{Problem Formulation}
\begin{figure}[tp]
    \centering
\includegraphics[width=0.35\textwidth]{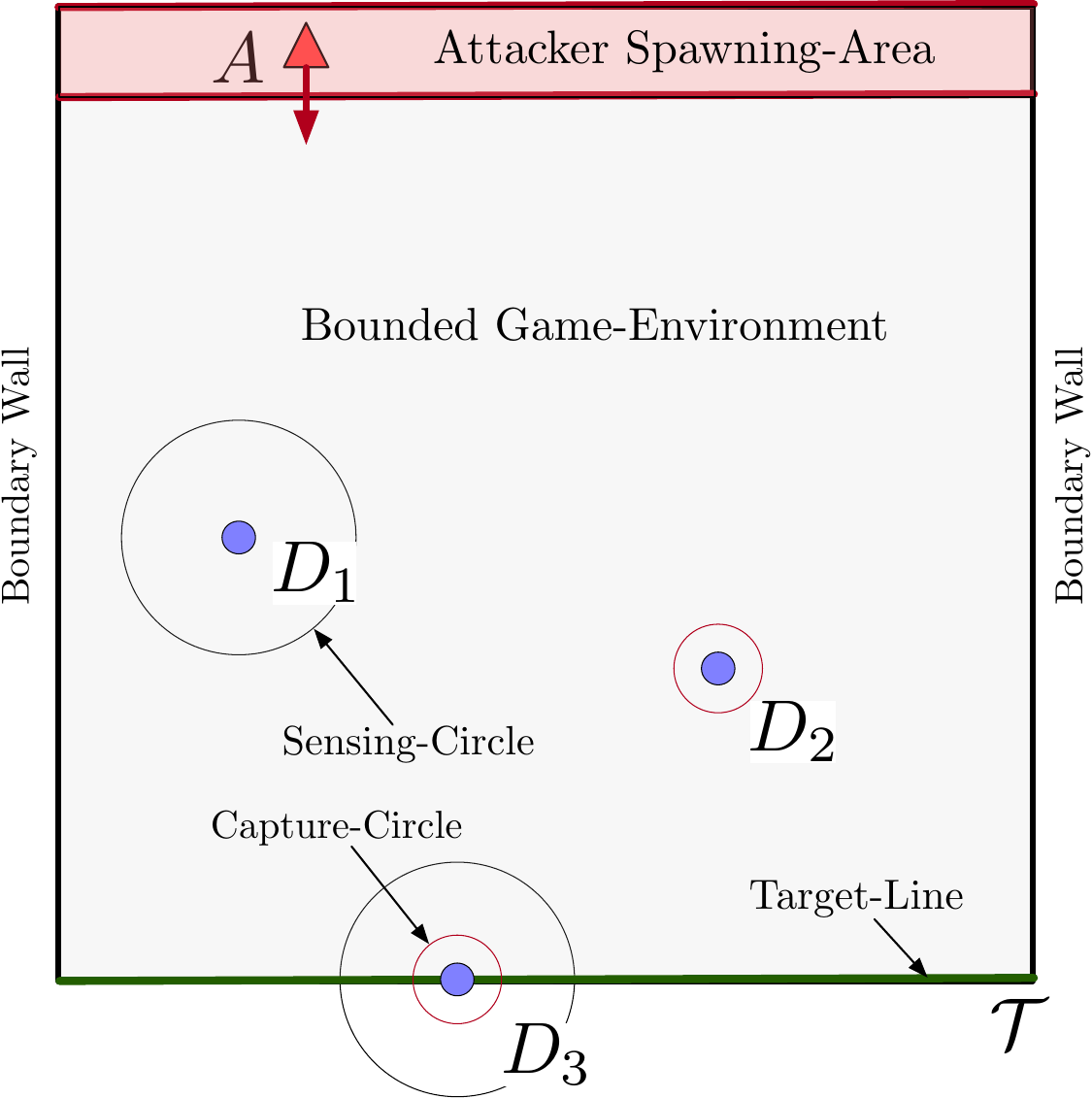}
    \caption{Illustration of the border-defense game environment. An attacker ($A$) spawns in the designated area at the top and attempts to reach the target line ($\mathcal{T}$) at the bottom. The defender team, composed of agents with varying sensing ($D_1$, $D_2$) and capture ($D_2$, $D_3$) capabilities, must coordinate to intercept the attacker.}
    \label{fig: game illustration}
\end{figure}

\noindent
We consider a target defense game  in a bounded, two-dimensional environment, $\Omega = [0, 1] \times [0, 1]$. The game involves a team of $n_d$ defenders, $D_i$ for $i \in \mathcal{I}=\{1, 2, ..., n_d\}$, and a single attacker, $A$. The defenders' objective is to protect a target boundary, $\mathcal{T}$, located at the bottom edge of the environment, defined as $\mathcal{T} = \{(x, y) \in \Omega \mid y = 0\}.
$
% \begin{align}
%     \mathcal{T} = \{(x, y) \in \Omega \mid y = 0\}.
% \end{align}

The  state of the game at time $t$ is given by $\mathbf{x}(t) \triangleq [\mathbf{x}_{D_1}^\top(t), \ldots, \mathbf{x}_{D_{n_d}}^\top(t), \mathbf{x}_A^\top(t)]^\top \in \mathbb{R}^{2(n_d+1)}$, which concatenates the positions of all agents, where $\mathbf{x}_k(t) = [x_k(t), y_k(t)]^\top$ denotes the position of agent $k \in \{D_1, D_2, \dots, D_{n_d}, A\}$. For notational brevity, we omit the time argument $(t)$ when the context is clear.
All agents are modeled with single-integrator dynamics:
% \begin{subequations}\label{eq:kinematics}
% \begin{align}
%     \dot{\mathbf{x}}_A &= v_A \mathbf{u}_A, \\
%     \dot{\mathbf{x}}_{D_i} &= v_{D_i} \mathbf{u}_{D_i},
% \end{align}
% \end{subequations}
\begin{align}\label{eq:kinematics}
    \dot{\mathbf{x}}_k = v_k \mathbf{u}_k, \quad k \in \{D_1, \ldots, D_{n_d}, A\},
\end{align}

where $v_k$ is the maximum speed of agent $k$, and $\mathbf{u}_k$ is its control input, constrained to the set of admissible unit vectors, $\mathcal{U} \triangleq \{\mathbf{u} \in \mathbb{R}^2 \mid \|\mathbf{u}\| \le 1\}$.

Each defender $D_i$ is characterized by a sensing radius $\rho_s^{(i)} \ge 0$ and a capture radius $\rho_c^{(i)} \ge 0$ , as illustrated in Figure~\ref{fig: game illustration}. We assume each defender is faster than the attacker, defining a speed ratio $\nu_i = v_{D_i}/v_A > 1$ for all $i \in \mathcal{I}$, which is a necessary condition for capture.

The game begins under bilateral information asymmetry, with neither side having knowledge of the opponent's initial positions. Both the attacker's and the defenders' starting positions are drawn from uniform distributions over separate, predefined regions $\Omega_A, \Omega_D \subset \Omega$:
\begin{align}
    \mathbf{x}_A(0) &\sim \text{Uniform}(\Omega_A), \label{eq:attacker_init}  \\
    \{\mathbf{x}_{D_i}(0)\}_{i \in \mathcal{I}} &\sim \text{Uniform}(\Omega_D). \label{eq:defender_init}
\end{align}
% where $ \Omega_A, \Omega_D \subset \Omega$. 
The game thus starts in a partial-information setting, which transitions to a perfect-information scenario  when the attacker is sensed at $t_s$:
\begin{align}\label{eq:sensing_time}
    t_s = \inf\{t \ge 0 \mid \exists i \in \mathcal{I} \text{ s.t. } \|\mathbf{x}_A - \mathbf{x}_{D_i}\| \le \rho_s^{(i)}\}.
\end{align}
% We assume that, after sensing, the attacker's state is known to all defenders for all $t \ge t_s$. 
We assume that, after sensing, full state information trivially becomes available to all agents for $t \ge t_s$.

We denote the defender and attacker policies as $\gamma_{D_i}$ and $\gamma_A$, respectively. Prior to sensing ($t < t_s$), each defender employs a policy based on the full-state information of all defenders:
% team-shared partial observations:
\begin{align}\label{eq:defender_policy_pre}
    \gamma_{D_i} : \mathcal{O}_{i} \to \mathcal{U}, \quad \text{for } i \in \mathcal{I},
\end{align}
where $\mathcal{O}_{i}$ represents the observation space of defender $i$, which includes its own position and the positions of other defenders. After sensing ($t \ge t_s$), the defender policies have access to the full state information:
\begin{align}
    \gamma_{D_i} : \mathbb{R}^{2(n_d+1)} \to \mathcal{U}.
\end{align}
Similarly, the attacker employs a two-stage policy:
\begin{align}\label{eq:attacker_policy}
    \gamma_A = \begin{cases}
        \mathbf{u}_{\text{direct}} = [0, -1]^\top & \text{if } t < t_s, \\
        \gamma_A^*(\mathbf{x}(t)) : \mathbb{R}^{2(n_d+1)} \to \mathcal{U} & \text{if } t \ge t_s,
    \end{cases}
\end{align}
where $\mathbf{u}_{\text{direct}}$ denotes the attacker's fixed policy before sensing, and $\gamma_A^*(\mathbf{x}(t))$ represents the attacker's optimal evasion strategy upon sensing. The direct-descent policy in Phase~I is justified by the bilateral information asymmetry: since the attacker has no sensing capability and no knowledge of defender positions, any deviation from the shortest path to $\mathcal{T}$ would only increase the time to reach the target without providing any informational benefit.

The game terminates at time $t_f$ when the state $\mathbf{x}(t_f)$ enters the terminal manifold $\mathcal{S} = \mathcal{S}_A \cup \mathcal{S}_D$, where:
\begin{subequations}
    \begin{align}
        \mathcal{S}_A &= \{\mathbf{x} \mid y_A = 0\}, \\
        \mathcal{S}_D &= \{\mathbf{x} \mid \exists i \in \mathcal{I}: 
        % \text{ with } \rho_c^{(i)} > 0,
        \|\mathbf{x}_A - \mathbf{x}_{D_i}\| \le \rho_c^{(i)}\}.
    \end{align}
\end{subequations}
The final time is thus $t_f = \inf\{t \ge 0 \mid \mathbf{x}(t) \in \mathcal{S}\}$.

This is a zero-sum game where the terminal payoff, $J$, is the attacker's distance from the target at the time of capture:
\begin{align}
    J = \max\{0, y_A(t_f)\}.
\end{align}
% Given the uncertainty in the players starting position, the defender team's objective is to employ a policy $\gamma_D = [\gamma_{D_1},\gamma_{D_2},\dots, \gamma_{D_{n_d}}]$, that maximizes the expected payoff, while the attacker's objective is to use a policy $\gamma_A$ to minimize it:
% \begin{align}\label{eq:expected_payoff}
%      J^* = \max_{\gamma_D} \min_{\gamma_A} \mathbb{E} [J],
% \end{align}
% where the expectation is taken over the distribution over $\mathbf{x}_A(0)$ and $\mathbf{x}_{D_i}(0)$. Since this is a zero-sum game, the saddle-point condition guarantees $\max_{\gamma_D} \min_{\gamma_A} \mathbb{E}[J] = \min_{\gamma_A} \max_{\gamma_D} \mathbb{E}[J]$, so the order of optimization is interchangeable; we adopt the maximizer-first form as our focus is on learning the defender policy while the attacker policy is fixed as described in \eqref{eq:attacker_policy}.

Given the uncertainty in the players' starting positions, the general zero-sum formulation is:
\begin{align}\label{eq:expected_payoff}
     J^* = \max_{\gamma_D} \min_{\gamma_A} \mathbb{E} [J],
\end{align}
where $\gamma_D = [\gamma_{D_1},\gamma_{D_2},\dots, \gamma_{D_{n_d}}]$ and the expectation is over initial conditions and any stochasticity in the policies. In this work, we fix the attacker policy as described in \eqref{eq:attacker_policy} and focus on learning the defender policy, reducing the problem to a one-sided optimization:
\begin{align}\label{eq:defender_objective}
     J^* = \max_{\gamma_D} \mathbb{E} [J \mid \gamma_A].
\end{align}

% \subsection{Team Composition}
% \noindent
% We evaluate defender teams with two distinct compositions based on their sensing and capture roles:
% \begin{itemize}
%     \item \textbf{Homogeneous Team:} All defenders are identical \emph{hybrid} agents, equipped with both sensing and capture capabilities, i.e., $\rho_s^{(i)} = \rho_s > 0$ and $\rho_c^{(i)} = \rho_c > 0$ for all $i \in \mathcal{I}$.
%     \item \textbf{Heterogeneous Team:} Defenders have specialized roles. The team consists of dedicated \emph{sensors} ($\rho_s^{(i)} > 0, \rho_c^{(i)} = 0$) and \emph{pursuers} ($\rho_s^{(i)} = 0, \rho_c^{(i)} > 0$), thus requiring explicit coordination between agents to achieve a capture. Figure~\ref{fig: game illustration} illustrates the border defense game environment with heterogeneous defender roles.
% \end{itemize}

\section{Game Theoretic Analysis}
\noindent
% In this section, we discuss analytical insights of the problem from a game-theoretic perspective. We first decompose the game into two distinct phases based on the information available to the agents. We then leverage the analytical solution of the second phase to develop our GT-assisted early termination mechanism for reinforcement learning.
We decompose the game into two phases based on available information and leverage the analytical solution of the second phase to develop GT-assisted early termination for RL.
\subsection{Phase Decomposition and Strategic Coupling}
\noindent
The target defense game can be naturally divided into two sequential phases:

% \begin{definition}[Phase I - Stochastic Search]
% The search phase spans the interval $[0, t_s)$ where:
% \begin{align}
%     \forall t \in [0, t_s): \quad \|\mathbf{x}_A(t) - \mathbf{x}_{D_i}(t)\| > \rho_s^{(i)}, \quad \forall i \in \mathcal{I}.
% \end{align}
% During this phase, the defender team plays the game under partial observability with observation space:
% \begin{equation}
% \mathcal{O}^{I}_{i} \subseteq \mathbb{R}^{2n_d}, 
% \qquad 
% o^i_t = \begin{bmatrix} 
% x_{D_1}^\top & x_{D_2}^\top & \cdots & x_{D_{n_d}}^\top
% \end{bmatrix}^\top.
% \end{equation}
% \end{definition}
\begin{definition}[Phase I - Stochastic Search]
The search phase spans $[0, t_s)$, during which no defender has sensed the attacker (cf.\ \eqref{eq:sensing_time}). Each defender's policy depends only on the positions of all defenders as defined in \eqref{eq:defender_policy_pre}.
\end{definition}

\begin{definition}[Phase II - Deterministic Pursuit]
The pursuit phase begins at $t_s$ when sensing occurs:
\begin{align}
    \exists i \in \mathcal{I}: \quad \|\mathbf{x}_A(t_s) - \mathbf{x}_{D_i}(t_s)\| = \rho_s^{(i)}.
\end{align}
For all $t \geq t_s$, defenders have access to the perfect state:
\begin{align}
    \mathcal{O}_{i}^{\text{II}} = \mathbf{x}(t) = [\mathbf{x}_{D_1}(t), \ldots, \mathbf{x}_{D_{n_d}}(t), \mathbf{x}_A(t)]^\top.
\end{align}
\end{definition}
 
%  Therefore, Phase I constitutes a stochastic optimization problem. In contrast, Phase II is a deterministic game of perfect information, for which an optimal defender policy can be derived. The optimal defender policy $\gamma_D^{\text{I}*}$ must maximize the expected probability of sensing given uncertainty in $\mathbf{x}_A(0)$:
%     \begin{align}
%         \gamma_D^{\text{I}*} = \arg\max_{\gamma_D} \mathbb{P}[t_s < \infty \mid \mathbf{x}_A(0) \sim \text{Uniform}(\Omega_A)].
%     \end{align}
% Given the state $\mathbf{x}(t_s)$ at sensing time $t_s$, the game outcome is deterministically computable through the Nash equilibrium of the perfect information pursuit game. Note that maximizing the sensing probability is related to, but not equivalent to, minimizing the expected sensing time; moreover, as discussed in Remark~1, the full objective must also account for the quality of the spatial configuration at $t_s$, not just whether or when sensing occurs.
Therefore, Phase I constitutes a stochastic optimization problem. In contrast, Phase II is a deterministic game of perfect information, for which an optimal defender policy can be derived. 
% A naive Phase~I objective would maximize the sensing probability alone; however, as formalized in Section~III-C, the full objective must jointly account for both sensing probability and configuration quality:
% The full Phase~I objective, formalized in Section~III-C, is:
% \begin{align}
%     \gamma_D^{\text{I}*} = \arg\max_{\gamma_D} \mathbb{E}\left[\mathbb{I}[t_s < \infty] \cdot J^*(\mathbf{x}(t_s))\right],
% \end{align}
% where the expectation is over both initial conditions $\mathbf{x}_A(0)$ and $\{\mathbf{x}_{D_i}(0)\}_{i \in \mathcal{I}}$. Given the state $\mathbf{x}(t_s)$ at sensing time $t_s$, the game outcome is deterministically computable through the Nash equilibrium of the perfect information pursuit game.
The full Phase~I objective, formalized in Section~III-C, is:
    \begin{align}\label{eq:gamma^I_D*}
        \gamma_D^{\text{I}*} = \arg\max_{\gamma_D} \mathbb{E}\left[\mathbb{I}_{t_s < t_f} \cdot J^*(\mathbf{x}(t_s))\right],
    \end{align}
where $\mathbb{I}_{t_s < t_f}$ is the indicator function that equals 1 if sensing occurs before game termination and 0 otherwise, and the expectation is over initial conditions and policy stochasticity. Note that cases where sensing occurs but the defenders cannot prevent a breach are handled by $J^*(\mathbf{x}(t_s)) = 0$.

% \textbf{Critical Insight:} 
% \begin{rem}[Search-Pursuit Trade-off]
% The optimal Phase I policy must balance two competing objectives:
% \begin{enumerate}
%     \item \textbf{Sensing Probability:} Maximize $\mathbb{P}[t_s < \infty]$ through effective spatial coverage.
%     \item \textbf{Pursuit Readiness:} Achieve sensing with configuration $\mathbf{x}(t_s)$ that maximizes the Phase~II Nash equilibrium payoff $J^*(\mathbf{x}(t_s))$.
% \end{enumerate}
% A myopic search strategy maximizing only sensing probability may result in poor pursuit configurations, yielding suboptimal overall performance despite successful sensing.
% \end{rem}
\begin{rem}[Search-Pursuit Trade-off]
The optimal Phase~I policy must balance two competing objectives: maximizing the probability of sensing before game termination, $\mathbb{P}[t_s < t_f]$, through effective spatial coverage, and achieving a configuration $\mathbf{x}(t_s)$ that maximizes the defender payoff. A myopic search strategy maximizing only sensing probability may result in poor pursuit configurations, yielding suboptimal overall performance despite successful sensing.
\end{rem}

\subsection{Analytical Solution for the Pursuit Phase}\label{subsec:apollonius}
\noindent
Once sensing occurs at time $t_s$, Phase II becomes a perfect information pursuit-evasion game. We leverage the geometric structure of the Apollonius Circle to compute the Nash equilibrium outcome analytically.

\subsubsection{Single Defender Case}
\begin{figure}[tp]
    \centering
    \includegraphics[width=0.4\textwidth]{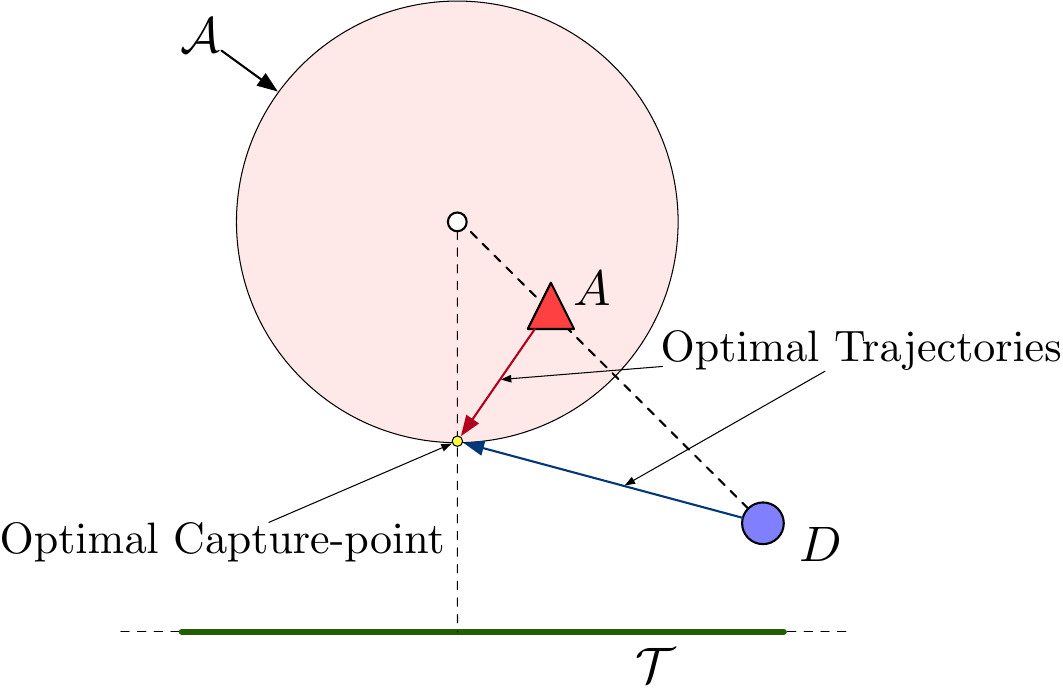}
    \caption{Geometric representation of the optimal capture strategy in the deterministic pursuit phase (Phase II). The Apollonius Circle ($\mathcal{A}$) defines the boundary of the attacker's dominance region. The optimal capture point is the location on the circle with the minimum y-coordinate, and the optimal trajectories for both agents converge at this point.}
    \label{fig: Apollonius_circle}
\end{figure}
\noindent
% For a single defender-attacker pair with speed ratio $\nu = v_D/v_A > 1$, the Apollonius Circle characterizes the set of points where the time-to-reach ratio equals the inverse speed ratio.
For a single defender-attacker pair with speed ratio $\nu = v_D/v_A > 1$, the Apollonius Circle $\mathcal{A}(\mathbf{x}_A, \mathbf{x}_D, \nu)$ is the locus of points $\mathbf{p}$ satisfying $\|\mathbf{p} - \mathbf{x}_A\| / \|\mathbf{p} - \mathbf{x}_D\| = 1/\nu$.
% \begin{definition}[Apollonius Circle]
% Given defender position $\mathbf{x}_D$ and attacker position $\mathbf{x}_A$ with speed ratio $\nu > 1$, the Apollonius Circle \cite{dorothy2024one} is:
% \begin{align}
%     \mathcal{A}(\mathbf{x}_A, \mathbf{x}_D, \nu) = \left\{\mathbf{p} \in \mathbb{R}^2 : \frac{\|\mathbf{p} - \mathbf{x}_A\|}{v_A} = \frac{\|\mathbf{p} - \mathbf{x}_D\|}{v_D}\right\}.
% \end{align}
% \end{definition}
% The geometric properties of this circle are:
% \begin{prop}[Circle Properties]
The center $\mathbf{c}$ and radius $r$ of $\mathcal{A}$ are given by:
\begin{subequations}
\begin{align}
    \mathbf{c} &= \frac{\nu^2 \mathbf{x}_A - \mathbf{x}_D}{\nu^2 - 1}, \label{eq:apo_center}\\
    r &= \frac{\nu \|\mathbf{x}_A - \mathbf{x}_D\|}{|\nu^2 - 1|}. \label{eq:apo_radius}
\end{align}
\end{subequations}

% \begin{lm}[Dominance Region \cite{dorothy2024one}]\label{lem:dominance}
% For a pursuit game with speed ratio $\nu > 1$, the closure of the Apollonius Circle $\text{cl}(\mathcal{A})$ constitutes the attacker's dominance region—the set of points the attacker can reach no later than the defender, which is given by the closed dominance disk
% \begin{align}\label{eq:dominance_disk}
%    \mathcal{D}\left(x_A, x_{D_i}, \nu_i\right)=\left\{p:\left\|p-c_i\right\| \leq r_i\right\}.
% \end{align}
% \end{lm}
\begin{lm}[Dominance Region \cite{Isaacs1965, dorothy2024one}]\label{lem:dominance}
The closed disk $\mathcal{D}(\mathbf{x}_A, \mathbf{x}_{D_i}, \nu_i) = \{\mathbf{p} : \|\mathbf{p} - \mathbf{c}_i\| \leq r_i\}$ is the attacker's dominance region—the set of points the attacker can reach no later than the defender.
\label{eq:dominance_disk}
\end{lm}

\begin{theorem}[Nash Equilibrium Payoff]\label{thm:nash}
For the Phase II pursuit game starting at configuration $\mathbf{x}(t_s)$, the Nash equilibrium payoff is:
\begin{align}\label{eq:nash_payoff_single}
    J^*(\mathbf{x}(t_s)) = \max\{0, c_y - r\},
\end{align}
where $c_y$ is the y-coordinate of the Apollonius Circle center and $r$ is its radius.
\end{theorem}

\begin{proof}
The result follows from Lemma~\ref{lem:dominance}: the attacker's optimal strategy is to move toward the lowest point of the dominance region, $\mathbf{p}^* = [c_x, c_y - r]^\top$, while the defender moves to intercept at $\mathbf{p}^*$. The payoff is clamped at 0 when $c_y - r < 0$, indicating the attacker reaches $\mathcal{T}$ before interception.
\end{proof}
% Under optimal play, the attacker aims to minimize its terminal y-coordinate while the defender seeks to maximize it. From Lemma~\ref{lem:dominance}, the attacker can reach any point in $\text{int}(\mathcal{A})$ before interception.

% Thus the optimal evasion policy $\gamma_A^*$ directs the attacker toward the capture point $\mathbf{p}^*$:
% \begin{align}
%     \gamma_A^*(\mathbf{x}(t)) = \frac{\mathbf{p}^* - \mathbf{x}_A(t)}{\|\mathbf{p}^* - \mathbf{x}_A(t)\|}.
% \end{align}
% where
% \begin{align}
%     \mathbf{p}^* = \arg\min_{\mathbf{p} \in \mathcal{D}} p_y,
% \end{align}
% and $\mathcal{D}$ is given in \eqref{eq:dominance_disk}. 

% Since $\mathcal{D}$ is a circular disk with center $\mathbf{c} = [c_x, c_y]^\top$ and radius $r$, the minimum y-coordinate occurs at:
% \begin{align}
%     \mathbf{p}^* = \mathbf{c} - r\mathbf{e}_y = [c_x, c_y - r]^\top.
% \end{align}

% If $c_y - r < 0$, the attacker reaches the target line ($y = 0$) before interception, yielding payoff 0. Otherwise, the defender intercepts at $y = c_y - r$.

% The defender's optimal strategy is to move directly toward $\mathbf{p}^*$, ensuring interception occurs exactly on the Apollonius Circle boundary, preventing the attacker from reaching any point with lower y-coordinate. Figure~\ref{fig: Apollonius_circle} illustrates the geometric structure of the Apollonius Circle and the optimal trajectories converging at the capture point.
% \end{proof}

\begin{rem}[Effect of Nonzero Capture Radius]
The Nash payoff in \eqref{eq:nash_payoff_single} assumes point capture ($\rho_c = 0$). When $\rho_c > 0$, the defender can intercept the attacker once $\|\mathbf{x}_A - \mathbf{x}_{D_i}\| \le \rho_c$, allowing capture to occur before the agents reach the point-capture meeting location predicted by the Apollonius Circle analysis. Consequently, \eqref{eq:nash_payoff_single} provides a conservative approximation of the true Phase~II payoff when $\rho_c > 0$.
\end{rem}

\subsubsection{Multiple Defender Case}
\noindent
For multiple defenders with capture capabilities, the attacker's dominance region becomes the intersection of individual Apollonius Circles.

\begin{figure}[t]
    \centering
    \includegraphics[width=0.40\textwidth]{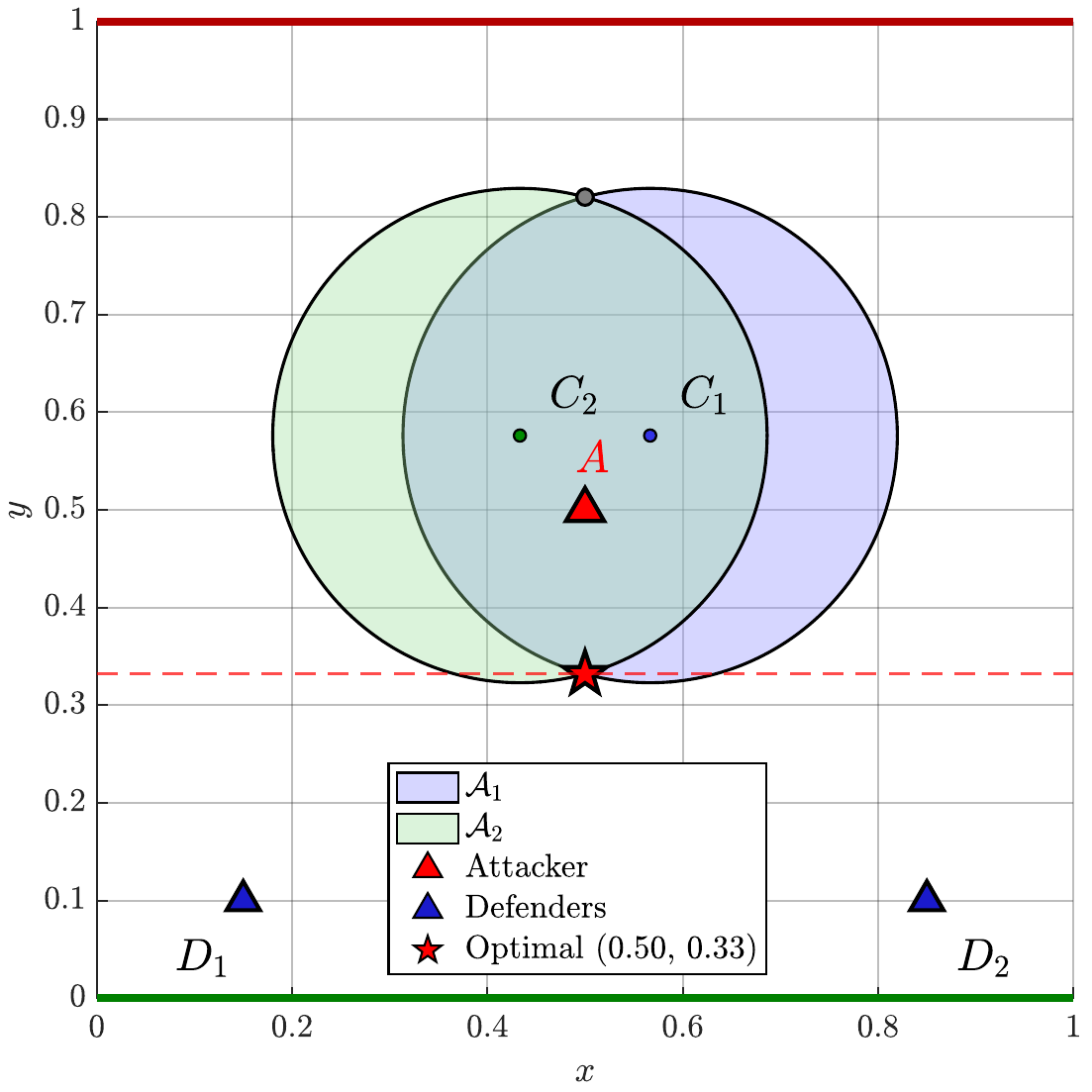}
    \caption{Optimal capture point computation for a multi-defender scenario. The attacker's reachable region is the intersection of the individual Apollonius Circles ($\mathcal{A}_1, \mathcal{A}_2$). The optimal interception point, marked by the red star, corresponds to the minimum y-coordinate within this intersection, representing the game-theoretic payoff.}
    \label{fig:two_defender}
\end{figure}
\begin{definition}[Multi-Defender Dominance Region]
Given $n_c$ defenders with capture capability at positions $\{\mathbf{x}_{D_i}\}_{i \in \mathcal{I}_c}$ and speed ratios $\{\nu_i\}_{i \in \mathcal{I}_c}$, the attacker's dominance region is:
\begin{align}
    \mathcal{R}_A = \bigcap_{i \in \mathcal{I}_c} 
    % \text{int}
    \mathcal{D}(\mathbf{x}_A, \mathbf{x}_{D_i}, \nu_i).
\end{align}
\end{definition}
Figure~\ref{fig:two_defender} illustrates this geometric relationship for a two-defender scenario, showing how the optimal capture point corresponds to the minimum y-coordinate within the intersection region.

\begin{prop}[Multi-Defender Nash Equilibrium]
For multiple defenders, the Nash equilibrium payoff is:
\begin{align}
    J^*(\mathbf{x}(t_s)) = \min_{\mathbf{p} \in \text{cl}(\mathcal{R}_A)} p_y. \label{eq:nash_payoff_multi}
\end{align}
\end{prop}
The computation of \eqref{eq:nash_payoff_multi} requires solving a constrained optimization problem:
\begin{align}
    \min_{p_x, p_y} \quad & p_y \\
    \text{s.t.} \quad & \|\mathbf{p} - \mathbf{c}_i\|^2 \leq r_i^2, \quad \forall i \in \mathcal{I}_c,
\end{align}
where $\mathbf{c}_i$ and $r_i$ are the center and radius of the $i$-th Apollonius Circle.

Figure~\ref{fig:contour_placement} demonstrates how the Nash equilibrium payoff varies with defender positioning, illustrating the importance of spatial configuration quality during the search phase.
\begin{rem}[Computational Complexity]
While the single-defender case admits a closed-form solution, the multi-defender case requires numerical optimization. However, the problem remains convex (minimizing a linear objective over an intersection of convex sets), ensuring efficient computation via standard convex optimization techniques.
\end{rem}

\begin{figure}[t]
    \centering
    \includegraphics[width=0.45\textwidth]{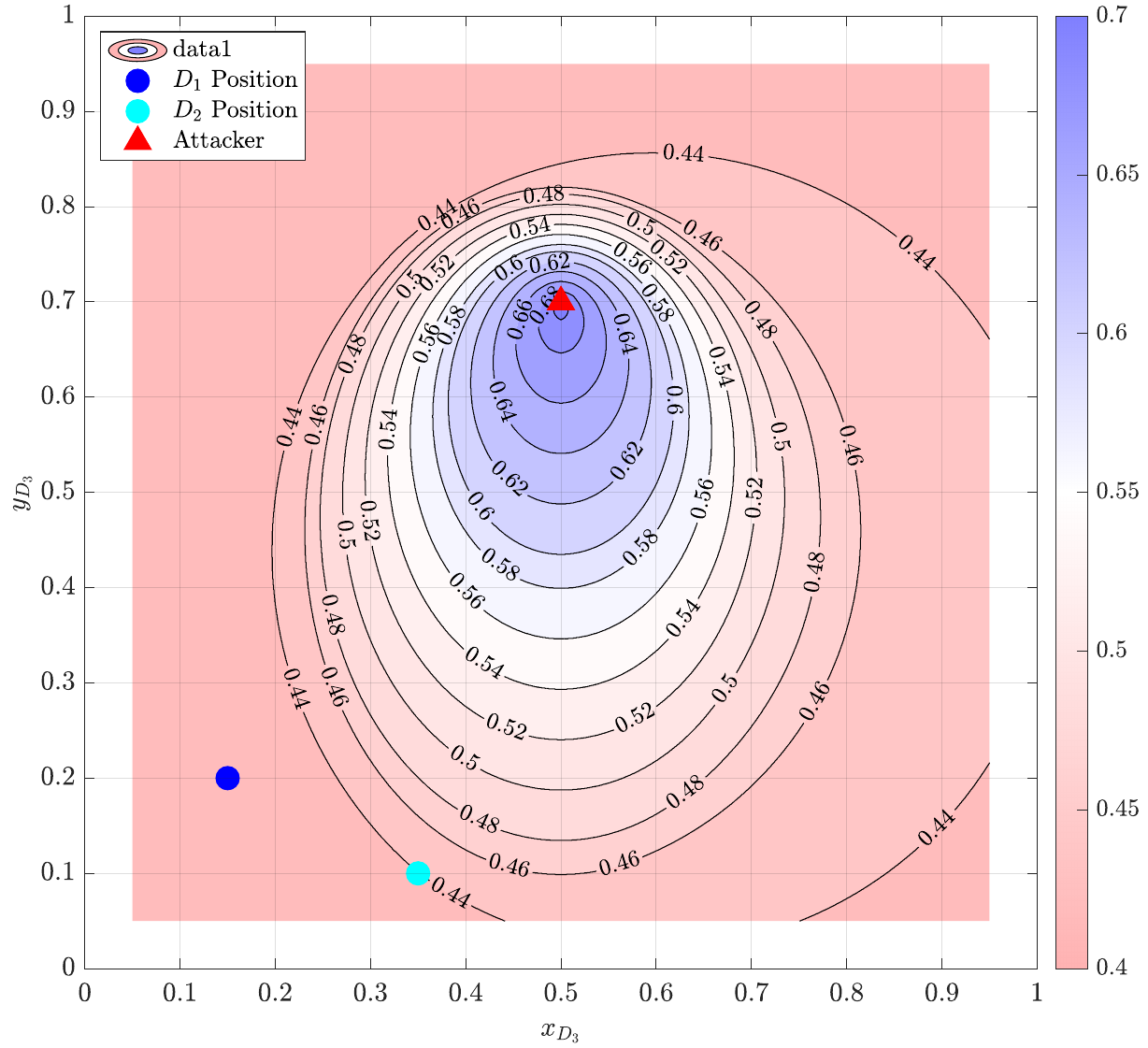}
    \caption{Value function level-set for optimal defender placement. The contours illustrate the Nash equilibrium payoff ($J^*$) as a function of a third defender's position, given two fixed defenders. Blue regions indicate configurations that yield a better payoff for the attacker, while red regions are more favorable for the defenders.}
    \label{fig:contour_placement}
\end{figure}

\subsection{GT-Assisted Reward Mechanism}
\noindent
The analytical solution for Phase II enables a training approach that eliminates the need to learn pursuit dynamics while maintaining theoretical optimality. Crucially, RL remains essential for learning the Phase~I search strategies, for which no analytical solution exists. 
% The GT-assisted mechanism only replaces the Phase~II simulation by providing analytically computed terminal rewards at sensing time. 
We formalize this mechanism and establish its relationship to full trajectory simulation.

\subsubsection{Early Termination Principle}
\noindent
Traditional reinforcement learning for target defense requires agents to learn both search and pursuit behaviors through extensive trial and error. Our GT-assisted approach leverages the analytical solution from Section~\ref{subsec:apollonius} to provide immediate, theoretically-grounded rewards at sensing time.

\begin{definition}[GT-Assisted Termination]
Upon sensing at time $t_s$ with configuration $\mathbf{x}(t_s)$, the GT-assisted mechanism:
\begin{enumerate}
    \item Computes the Nash equilibrium payoff $J^*(\mathbf{x}(t_s))$ using equation \eqref{eq:nash_payoff_multi}
    \item Terminates the episode immediately
    \item Assigns terminal reward $R = J^*(\mathbf{x}(t_s))$ to all defenders
\end{enumerate}
\end{definition}

This approach contrasts with standard RL training requiring agents to learn pursuit strategies through exploration.
% Algorithm~\ref{alg:gt_assist} describes the implementation of the GT-assisted early termination of the game episodes.

% \subsubsection{Theoretical Equivalence}
% \noindent
% The GT-assisted rewards provide the same learning signal as optimal play in the full game. Specifically, let $R^{\text{GT}}$ denote the reward from GT-assisted termination and $R^{\text{full}}$ denote the reward from simulating optimal pursuit to completion. Then:
% \begin{equation}
% R^{\text{GT}}(\mathbf{x}(t_s)) = \mathbb{E}[R^{\text{full}}(\mathbf{x}(t_s))|\text{optimal play}].
% \label{eq:reward_equivalence}
% \end{equation}

% % \begin{proof}
% By Theorem~\ref{thm:nash}, the Nash equilibrium payoff $J^*(\mathbf{x}(t_s))$ represents the equilibrium payoff under optimal play from both sides. The GT-assisted mechanism directly assigns this value as the reward:
% \begin{align}\label{eq:reward_equiv}
%     R^{\text{GT}}(\mathbf{x}(t_s)) = J^*(\mathbf{x}(t_s)).
% \end{align}

% Under optimal play in the full game, the terminal state yields:
% \begin{align}
%     R^{\text{full}}(\mathbf{x}(t_s)) = y_A(t_f^*),
% \end{align}
% where $t_f^*$ is the optimal termination time. By definition of the Nash equilibrium:
% \begin{align}
%     y_A(t_f^*) = J^*(\mathbf{x}(t_s)).
% \end{align}

% Therefore, $R^{\text{GT}} = R^{\text{full}}$ under optimal play. For suboptimal play, the GT-assisted reward provides an upper bound on achievable performance, preventing reward exploitation through poor pursuit execution.
% % \end{proof}

% \subsubsection{Theoretical Equivalence}
\subsubsection{Undiscounted Payoff Equivalence}
\noindent
Let $R^{\text{GT}}$ and $R^{\text{full}}$ denote the rewards from GT-assisted termination and full optimal pursuit simulation, respectively. Then:
\begin{equation}
R^{\text{GT}}(\mathbf{x}(t_s)) = \mathbb{E}[R^{\text{full}}(\mathbf{x}(t_s))|\text{optimal play}].
\label{eq:reward_equivalence}
\end{equation}
By Theorem~\ref{thm:nash}, the GT-assisted mechanism assigns:
\begin{align}\label{eq:reward_equiv}
    R^{\text{GT}}(\mathbf{x}(t_s)) = J^*(\mathbf{x}(t_s)),
\end{align}
which equals $y_A(t_f^*)$ under optimal play by both sides, where $t_f^*$ is the optimal termination time. Note that while the undiscounted terminal payoffs are equivalent, the discounted returns generally differ because the GT-assisted reward is assigned at $t_s < T_f$. Thus, under discounting, the GT-assisted formulation induces a modified learning objective that places additional value on earlier sensing, which is aligned with our goal of encouraging efficient search behavior. For suboptimal play, the GT-assisted reward provides an upper bound, preventing reward exploitation through poor pursuit execution.

\subsubsection{Computational and Learning Advantages}
\noindent
The GT-assisted mechanism provides several key advantages:

\begin{prop}[Computational Efficiency] The GT-assisted approach reduces the computational cost of each episode by eliminating the simulation of the pursuit phase. The number of saved steps per episode, $T_p$, can be estimated by considering a direct, one-dimensional pursuit:
\begin{equation}
    T_p \approx \frac{(\rho_s - \rho_c) / (v_D - v_A)}{\Delta t}
\end{equation}
where the numerator represents the approximate time required to close the distance from sensing to capture at the maximum relative speed, and the denominator converts this time into discrete simulation steps.
\end{prop}
This shorter horizon also improves sample efficiency by removing the need to explore pursuit strategies.
% \begin{prop}[Sample Efficiency]
% By eliminating pursuit phase exploration, the GT-assisted mechanism reduces sample complexity from $\mathcal{O}(|\mathcal{S}_I| \cdot |\mathcal{S}_{II}|)$ to $\mathcal{O}(|\mathcal{S}_I|)$, where $\mathcal{S}_I$ and $\mathcal{S}_{II}$ are the state spaces of Phases I and II respectively.
% \end{prop}

% \subsubsection{Implementation Algorithm}
% Algorithm~\ref{alg:gt_assist} describes the implementation of the GT-assisted early termination in second phase of the game.
\begin{algorithm}[tp]
\caption{GT-Assisted Episode Termination}
\label{alg:gt_assist}
\begin{algorithmic}[1]
\State \textbf{Input:} Sensing configuration $\mathbf{x}(t_s)$, speed ratios $\{\nu_i\}$
\State \textbf{Output:} Terminal reward $R$

\State // Extract positions
\State $\mathbf{x}_A \gets$ AttackerPosition($\mathbf{x}(t_s)$)
\State $\{\mathbf{x}_{D_i}\} \gets$ DefenderPositions($\mathbf{x}(t_s)$)

\State // Compute Apollonius Circles
\For{each defender $i$ with capture capability}
    \State $\mathbf{c}_i \gets (\nu_i^2 \mathbf{x}_A - \mathbf{x}_{D_i})/(\nu_i^2 - 1)$
    \State $r_i \gets \nu_i \|\mathbf{x}_A - \mathbf{x}_{D_i}\|/|\nu_i^2 - 1|$
\EndFor

\State // Solve for Nash equilibrium
\If{single defender}
    \State $J^* \gets \max\{0, c_y - r\}$
\Else
    \State $J^* \gets$ SolveConvexOptimization($\{\mathbf{c}_i, r_i\}$)
\EndIf

\State \Return $R = J^*$
\end{algorithmic}
\end{algorithm}

\subsubsection{Impact on Learned Behaviors}
\noindent
The GT-assisted mechanism shapes the learned search strategies through the objective defined in \eqref{eq:gamma^I_D*}: the policy must learn to position defenders not just for sensing, but for configurations that enable effective pursuit—despite never experiencing pursuit during training.

% \subsubsection{Impact on Learned Behaviors}
% \noindent
% The GT-assisted mechanism fundamentally shapes the learned search strategies.
% 
% \begin{theorem}[Implicit Configuration Learning]
% Under GT-assisted training, the optimal Phase I policy $\gamma_D^*$ maximizes:
% \begin{align}\label{eq:implicit_objective}
%     \gamma_D^* = \arg\max_{\gamma_D} \mathbb{E}
%     \left[\mathbb{I}_{t_s < t_f} \cdot J^*(\mathbf{x}(t_s))\right],
% \end{align}
% where $\mathbb{I}_{t_s < t_f}$ is as defined in Section~III-A. This optimization implicitly encodes both sensing probability and configuration quality objectives.

% Therefore, maximizing $\mathcal{J}(\gamma_D)$ implicitly balances: (i) maximizing sensing probability $\mathbb{P}[t_s < \infty]$, and (ii) achieving favorable configurations that maximize $J^*(\mathbf{x}(t_s))$ when sensing occurs. The policy must learn to position defenders not just for sensing, but for configurations that enable effective pursuit—despite never experiencing pursuit during training.
% \end{rem}

\section{MULTI-AGENT REINFORCEMENT LEARNING APPROACH}

% \subsection{Decentralized Partially Observable MDP Formulation}
% \noindent
% % To learn effective search strategies for Phase I while leveraging analytical solutions for Phase II, w
% In this section, we formulate the target defense problem as a Decentralized Partially Observable Markov Decision Process (Dec-POMDP)

% % The Dec-POMDP
% defined by the tuple:
\subsection{Decentralized Partially Observable MDP Formulation}
\noindent
We formulate the problem as a Decentralized Partially Observable Markov Decision Process (Dec-POMDP) defined by the tuple:
\begin{equation}
\langle \mathcal{I}, \mathcal{S}, \{\mathcal{A}_i\}_{i \in \mathcal{I}}, \{\mathcal{O}_i\}_{i \in \mathcal{I}}, T, \{R_i\}_{i \in \mathcal{I}}, \gamma \rangle,
\end{equation}
where:
\begin{itemize}
\item $\mathcal{I} = \{1, 2, \ldots, n_d\}$: defender index set;
    % \item $\mathcal{S} = \Omega^{n_d+1} \subset \mathbb{R}^{2(n_d+1)}$: global state space with $\mathbf{s}_t = [\mathbf{x}_{D_1}^\top(t), \ldots, \mathbf{x}_{D_{n_d}}^\top(t), \mathbf{x}_A^\top(t)]^\top$;
    \item $\mathcal{S} = \Omega^{n_d+1} \subset \mathbb{R}^{2(n_d+1)}$: global state space with $\mathbf{s}_t \equiv \mathbf{x}(t)$;
    \item $\mathcal{A}_i \subset \mathbb{R}$: action space for defender $i$;
    \item $\mathcal{O}_i$: phase-dependent observation space as defined in Section II;
    \item $T: \mathcal{S} \times \mathcal{A}_1 \times \cdots \times \mathcal{A}_{n_d} \to \Delta(\mathcal{S})$: transition function governed by \eqref{eq:kinematics} and \eqref{eq:attacker_policy};
    \item $R_i: \mathcal{S} \to \mathbb{R}$: reward function for defender $i$;
    \item $\gamma \in [0,1)$: discount factor.
\end{itemize}

% Each defender $i$ learns a parameterized policy $\pi_{\theta_i}: \mathcal{O}_i \to \Delta(\mathcal{A}_i)$ that maps observations to probability distributions over actions. These learned policies correspond to the defender strategies $\gamma_{D_i}$ from the game-theoretic formulation in Section II. The team's objective is to learn a joint policy $\boldsymbol{\pi}_{\boldsymbol{\theta}} = \{\pi_{\theta_1}, \ldots, \pi_{\theta_{n_d}}\}$ that maximizes the expected cumulative discounted reward:
Each defender $i$ learns a policy $\pi_{\theta_i}: \mathcal{O}_i \to \Delta(\mathcal{A}_i)$, corresponding to the strategy $\gamma_{D_i}$ from Section II. The team objective is to learn a joint policy $\boldsymbol{\pi}_{\boldsymbol{\theta}} = \{\pi_{\theta_1}, \ldots, \pi_{\theta_{n_d}}\}$ maximizing the expected cumulative discounted reward:
\begin{equation}
J(\boldsymbol{\theta}) = \mathbb{E}_{\boldsymbol{\pi}_{\boldsymbol{\theta}}} \left[ \sum_{t=0}^{T_f} \gamma^t R_t \mid \mathbf{s}_0 \sim \rho_0 \right],
\label{eq:rl_objective}
\end{equation}
% where $T_f$ is the episode termination time, $R_t$ is the shared team reward at time $t$, and $\rho_0$ represents the initial state distribution defined by \eqref{eq:attacker_init}--\eqref{eq:defender_init}.
where $T_f$ is the episode termination time, $R_t$ is the shared team reward at time $t$, and $\rho_0$ represents the initial state distribution defined by \eqref{eq:attacker_init}--\eqref{eq:defender_init}. We note that the learned policies are not guaranteed to converge to a global optimum, nor are they directly interpretable; however, the GT-assisted reward mechanism ensures that the pursuit phase outcome is always evaluated optimally.

\subsection{Observation and Action Spaces}
\noindent
\subsubsection{Action Space}
\noindent
Each defender $i$ controls its heading through a continuous action:
\begin{equation}
a_t^i \in \mathcal{A}_i = [-1, 1] \subset \mathbb{R},
\label{eq:action_space}
\end{equation}
% which represents a normalized heading angle. This action maps to the control input $\mathbf{u}_{D_i}(t)$ in \eqref{eq:kinematics} via:
mapping to the control input $\mathbf{u}_{D_i}(t)$ in \eqref{eq:kinematics} via:
\begin{equation}
\mathbf{u}_{D_i}(t) = \begin{bmatrix} \cos(a_t^i \cdot \pi) \\ \sin(a_t^i \cdot \pi) \end{bmatrix},
\label{eq:action_to_control}
\end{equation}
ensuring $\|\mathbf{u}_{D_i}(t)\| = 1$ as required by the dynamics.

% \subsubsection{Phase-Dependent Observations}
% \noindent
% As established in Section II, the game exhibits information asymmetry that transitions at the sensing time $t_s$ defined in \eqref{eq:sensing_time}. This structure induces phase-dependent observation spaces for each defender.

% \subsubsection{Phase-Dependent Observations}
% \noindent
% \textbf{Phase I (Search Phase, $t < t_s$):} Before sensing, defender $i$ observes its own position and those of teammate defenders:
\subsubsection{Phase-Dependent Observations}
\noindent
The information asymmetry at sensing time $t_s$ induces phase-dependent observation spaces for each defender.

\textbf{Phase I (Search Phase, $t < t_s$):} Before sensing, defender $i$ observes its own position and those of teammate defenders:
\begin{equation}
\mathbf{o}_t^i \in \mathcal{O}_i^{\text{I}} 
% = \Omega^{n_d} 
\subset \mathbb{R}^{2n_d},
\label{eq:obs_phase1}
\end{equation}
where
\begin{equation}
\mathbf{o}_t^i = [\mathbf{x}_{D_1}^\top(t), \ldots, \mathbf{x}_{D_{n_d}}^\top(t)]^\top.
\label{eq:obs_phase1_form}
\end{equation}
The attacker position $\mathbf{x}_A(t)$ is not observable during this phase. Defenders share their positions via communication, so the partial observability is solely due to the unknown attacker state.

\textbf{Phase II (Pursuit Phase, $t \geq t_s$):} Upon sensing, the attacker's position becomes observable to all defenders:
\begin{equation}
\mathbf{o}_t^i \in \mathcal{O}_i^{\text{II}} = \Omega^{n_d+1} \subset \mathbb{R}^{2(n_d+1)},
\label{eq:obs_phase2}
\end{equation}
where
\begin{equation}
\mathbf{o}_t^i = [\mathbf{x}_{D_1}^\top(t), \ldots, \mathbf{x}_{D_{n_d}}^\top(t), \mathbf{x}_A^\top(t)]^\top = \mathbf{s}_t.
\label{eq:obs_phase2_form}
\end{equation}
% In this phase, observations coincide with the global state, and the game becomes fully observable to the defender team. Note that pursuers with $\rho_s^{(i)} = 0$ (capture-only role) also gain access to the attacker's position once any sensor detects it.
% The observation function for defender $i$ is therefore:
% Pursuers with $\rho_s^{(i)} = 0$ also gain access once any sensor detects the attacker. 
Pursuers with $\rho_s^{(i)} = 0$ also gain access to the full state once any sensor detects the attacker.
% The observation function for defender $i$ is:
% \begin{equation}
% \Omega_i(\mathbf{s}_t) = \begin{cases}
% [\mathbf{x}_{D_1}^\top(t), \ldots, \mathbf{x}_{D_{n_d}}^\top(t)]^\top & \text{if } t < t_s, \\
% \mathbf{s}_t & \text{if } t \geq t_s.
% \end{cases}
% \label{eq:observation_function}
% \end{equation}

\subsection{Reward Structure and Terminal Conditions}
\noindent
We evaluate two reward structures that differ in their terminal conditions and reward assignment mechanisms, enabling a comparative analysis of GT-assisted early termination versus standard end-to-end learning.

\subsubsection{GT-Assisted Reward}
\noindent
In the GT-assisted approach, episodes terminate immediately upon sensing at time $t_s$, and the terminal reward is computed using the Nash equilibrium payoff from \eqref{eq:nash_payoff_multi}:
\begin{equation}
R_t^i = \begin{cases}
0 & \text{if } t < t_s, \\
J^*(\mathbf{x}(t_s)) & \text{if } t = t_s,
\end{cases}
\quad \forall i \in \mathcal{I},
\label{eq:reward_gt}
\end{equation}
where $J^*(\mathbf{x}(t_s))$ is the analytically computed Nash equilibrium payoff given by \eqref{eq:nash_payoff_single} for a single defender or \eqref{eq:nash_payoff_multi} for multiple defenders. The terminal condition is:
\begin{equation}
\mathcal{S}_{\text{term}}^{\text{GT}} = \{\mathbf{s}_t \in \mathcal{S} \mid \exists i \in \mathcal{I} : \|\mathbf{x}_A(t) - \mathbf{x}_{D_i}(t)\| \leq \rho_s^{(i)}\}.
\label{eq:terminal_gt}
\end{equation}
From~\eqref{eq:reward_equiv}, this reward is equivalent to the expected outcome under optimal play in Phase II, eliminating the need for defenders to learn pursuit dynamics while providing theoretically grounded feedback for search strategy optimization.

\subsubsection{Standard Reward}
\noindent
In the standard approach, episodes continue through both search and pursuit phases until either capture or target breach occurs. The terminal reward is:
\begin{equation}
R_t^i = \begin{cases}
0 & \text{if } t < T_f, \\
\max\{0, y_A(T_f)\} & \text{if } t = T_f,
\end{cases}
\quad \forall i \in \mathcal{I},
\label{eq:reward_standard}
\end{equation}
where $T_f$ is the episode termination time and $y_A(T_f)$ is the attacker's final $y$-coordinate. The terminal condition is:
\begin{equation}
\begin{split}
\mathcal{S}_{\text{term}}^{\text{std}} = \{\mathbf{s}_t \in \mathcal{S} \mid y_A(t) \leq 0 \text{ or } \\
\quad \exists i \in \mathcal{I}_c : \|\mathbf{x}_A(t) - \mathbf{x}_{D_i}(t)\| \leq \rho_c^{(i)}\},
\label{eq:terminal_std}
\end{split}
\end{equation}
where $\mathcal{I}_c = \{i \in \mathcal{I} \mid \rho_c^{(i)} > 0\}$ is the set of defenders with capture capability.

\subsubsection{Shared Reward Structure}
\noindent
In both approaches, all defenders receive identical rewards ($R_t^i = R_t$ for all $i \in \mathcal{I}$) to encourage cooperative behavior and team coordination. This shared reward structure is critical for learning coordinated search patterns in Phase I. While individual reward schemes could reduce credit misattribution for a single agent's actions, they risk incentivizing greedy behaviors that undermine team coordination---particularly in the heterogeneous setting where the sensor's contribution is only realized through the pursuer's capture. The shared structure ensures all agents are aligned toward the common team objective.

\subsection{Multi-Agent Proximal Policy Optimization}
\noindent
% We employ Multi-Agent Proximal Policy Optimization (MAPPO) \cite{yu2022surprising} as implemented in the BenchMARL framework \cite{bettini2024benchmarl} for policy learning. MAPPO extends the single-agent PPO algorithm to multi-agent settings through centralized training with decentralized execution (CTDE).
We employ MAPPO \cite{yu2022surprising} via the BenchMARL framework \cite{bettini2024benchmarl}, which extends PPO to multi-agent settings through centralized training with decentralized execution (CTDE).

\subsubsection{Policy and Value Networks}
\noindent
Each defender $i$ maintains a policy network $\pi_{\theta_i}: \mathcal{O}_i \to \Delta(\mathcal{A}_i)$ outputting a Gaussian distribution:
\begin{equation}
a_t^i \sim \mathcal{N}(\mu_{\theta_i}(\mathbf{o}_t^i), \sigma_{\theta_i}(\mathbf{o}_t^i)),
\label{eq:policy_gaussian}
\end{equation}
% where $\mu_{\theta_i}$ and $\sigma_{\theta_i}$ are neural networks that produce the mean and standard deviation. Sampled actions are clamped to the admissible range $[-1, 1]$ to satisfy the action space constraints in \eqref{eq:action_space}.

% A centralized critic network $V_\phi: \mathcal{S} \to \mathbb{R}$ estimates the state value function using global state information during training:
where $\mu_{\theta_i}$ and $\sigma_{\theta_i}$ are the mean and standard deviation networks, with sampled actions clamped to $[-1, 1]$. A centralized critic $V_\phi: \mathcal{S} \to \mathbb{R}$ estimates the state value using global information during training:

\begin{equation}
V_\phi(\mathbf{s}_t) \approx \mathbb{E}\left[\sum_{k=t}^{T_f} \gamma^{k-t} R_k \mid \mathbf{s}_t\right],
\label{eq:value_function}
\end{equation}
where $R_k$ denotes the shared team reward at time $k$ (as all defenders receive identical rewards).

At execution time, each defender acts based only on local observations $\mathbf{o}_t^i$, enabling decentralized execution.

\subsubsection{Training Objective}
\noindent
% The policy for each defender is updated by maximizing the clipped surrogate objective:
Each policy is updated by maximizing the clipped surrogate objective:
\begin{equation}
L^{\text{CLIP}}(\theta_i) = \mathbb{E}_t\left[\min\left(r_t^i(\theta_i)\hat{A}_t^i, \text{clip}(r_t^i(\theta_i), 1-\epsilon, 1+\epsilon)\hat{A}_t^i\right)\right],
\label{eq:ppo_objective}
\end{equation}
% where $r_t^i(\theta_i) = \frac{\pi_{\theta_i}(a_t^i|\mathbf{o}_t^i)}{\pi_{\theta_i^{\text{old}}}(a_t^i|\mathbf{o}_t^i)}$ is the probability ratio, $\epsilon$ is the clip parameter, and $\hat{A}_t^i$ is the advantage estimate computed using Generalized Advantage Estimation (GAE) \cite{schulman2015high}:
where $r_t^i(\theta_i) = \pi_{\theta_i}(a_t^i|\mathbf{o}_t^i) / \pi_{\theta_i^{\text{old}}}(a_t^i|\mathbf{o}_t^i)$ is the probability ratio, $\epsilon$ is the clip parameter, and $\hat{A}_t^i$ is computed via GAE \cite{schulman2015high}:
\begin{equation}
\hat{A}_t^i = \sum_{l=0}^{\infty} (\gamma \lambda)^l \delta_{t+l},
\label{eq:gae}
\end{equation}
with $\delta_t = R_t + \gamma V_\phi(\mathbf{s}_{t+1}) - V_\phi(\mathbf{s}_t)$ being the temporal difference error.

\section{Experiments}
\noindent
% Experiments\footnote{Code and experimental data are available at \url{https://github.com/das-goutam/benchmarl-target-defense}} were conducted on CPU-based infrastructure using the BenchMARL \cite{bettini2024benchmarl} framework with VMAS \cite{bettini2022vmas} backend for vectorized simulation. We employed Multi-Agent PPO with 20 parallel environments, training for 1.2-1.8M frames (150-300 iterations) over 3 seeds for statistical significance.

% % \subsubsection{Environment Parameters}
% Table \ref{tab:env_params} summarizes the key environment parameters used in our experiments.

Experiments\footnote{Code and experimental data are available at \url{https://github.com/das-goutam/benchmarl-target-defense}} used BenchMARL \cite{bettini2024benchmarl} with VMAS \cite{bettini2022vmas} for vectorized simulation, employing MAPPO with 20 parallel environments, training for 1.2--1.8M frames over 3 seeds. Table~\ref{tab:env_params} summarizes the environment parameters.
% \daigo
{Homogeneous defender team was tested with the sensing radius  $\rho_s = 0.3$ for 1v1 (1 defender, 1 attacker) and $\rho_s = 0.15$ for 3v1 (3 defenders, 1 attacker) and the capture radius of $\rho_c = 0.07$.}
\begin{table}[t]
\centering
\caption{Environment Parameters}
\label{tab:env_params}
\begin{tabular}{ll}
\hline
\textbf{Parameter} & \textbf{Value} \\
\hline
Domain & $\Omega = [0,1] \times [0,1]$ \\
Target boundary & $\mathcal{T} = \{(x,y) \in \Omega \mid y = 0\}$ \\
Attacker spawn region & $\Omega_A = [0,1] \times [0.9, 1.0]$ \\
% Defender spawn region & $\Omega_D = [0,1] \times [0, 0.1]$ \\
Speed ratio & $\nu = v_{D_i}/v_A = 3.33$ \\
Sensing radius & $\rho_s \in \{0.15, 0.3, 0.35\}$  \\
Capture radius & $\rho_c = 0.07$ \\
Time step & $\Delta t = 0.05$ \\
Max episode length & $T_{\max} = 2/(v_A \Delta t)$  steps \\
\hline
\end{tabular}
\end{table}

% We evaluate two team compositions:

% \textbf{Homogeneous Teams:} 

% All defenders possess both sensing ($\rho_s = 0.3$ for 1v1 (1 defender, 1 attacker) and $\rho_s = 0.15$ for 3v1 (3 defenders, 1 attacker) configuration) and capture ($\rho_c = 0.07$) capabilities. 
% We test 1v1 and 3v1 configurations.

% \textbf{Heterogeneous Teams:} 2v1 configuration with sensor ($\rho_s^{(1)} = 0.35, \rho_c^{(1)} = 0$) and pursuer ($\rho_s^{(2)} = 0, \rho_c^{(2)} = 0.07$), requiring explicit coordination.

\subsection{Team Performance}
\noindent
In scenarios with homogeneous defenders, the GT-assisted approach consistently outperforms the end-to-end baseline in both single-defender (1v1) and multi-defender (3v1) configurations (Figs.~\ref{fig:gt_vs_e2e_1v1} \& \ref{fig:gt_vs_e2e_3v1}, respectively). For the 1v1 case, the GT-assisted policy converges rapidly to a mean reward of $\mu=0.605$, a $10 \%$ improvement over the baseline's $\mu=0.545$. This advantage is maintained in the 3v1 setting, where the GT-assisted method achieves a final reward of $\mu=0.745$, representing a $18.9 \%$ increase over the baseline's $\mu=0.626$. While the end-to-end approach could in principle converge to similar performance given sufficient training, it must simultaneously learn both search and pursuit behaviors, resulting in a fundamentally larger exploration space that hinders convergence in practice.

\begin{figure}[bp]
    \centering
    \includegraphics[width=0.45\textwidth]{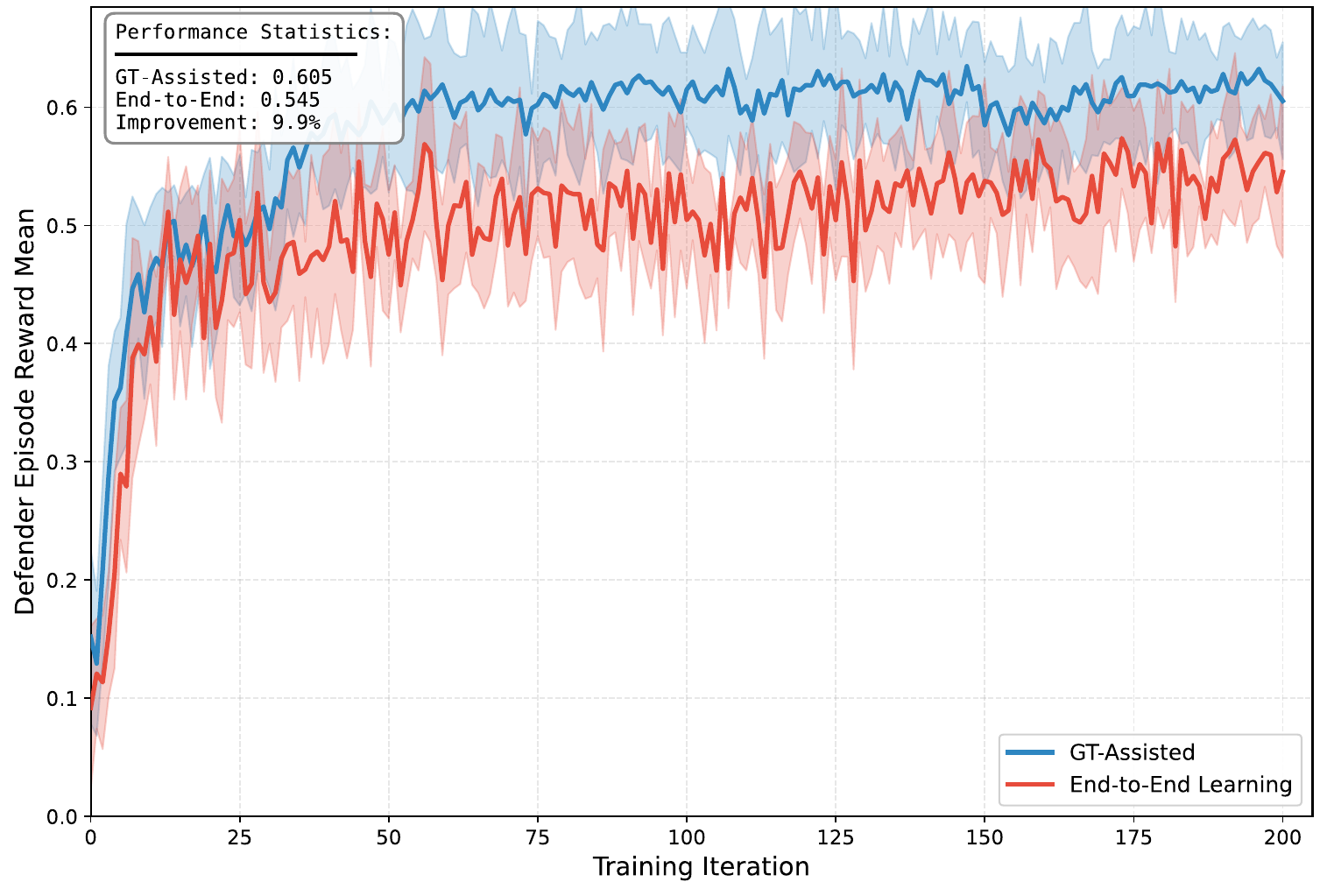}
    \caption{Training performance in the 1v1 homogeneous scenario. The GT-assisted policy (blue) converges faster and to a significantly higher mean reward ($\mu=0.605$) than the end-to-end baseline (red, $\mu=0.545$), demonstrating a 10\% performance improvement.}
    \label{fig:gt_vs_e2e_1v1}
\end{figure}
\begin{figure}[btp]
    \centering
    \includegraphics[width=0.45\textwidth]{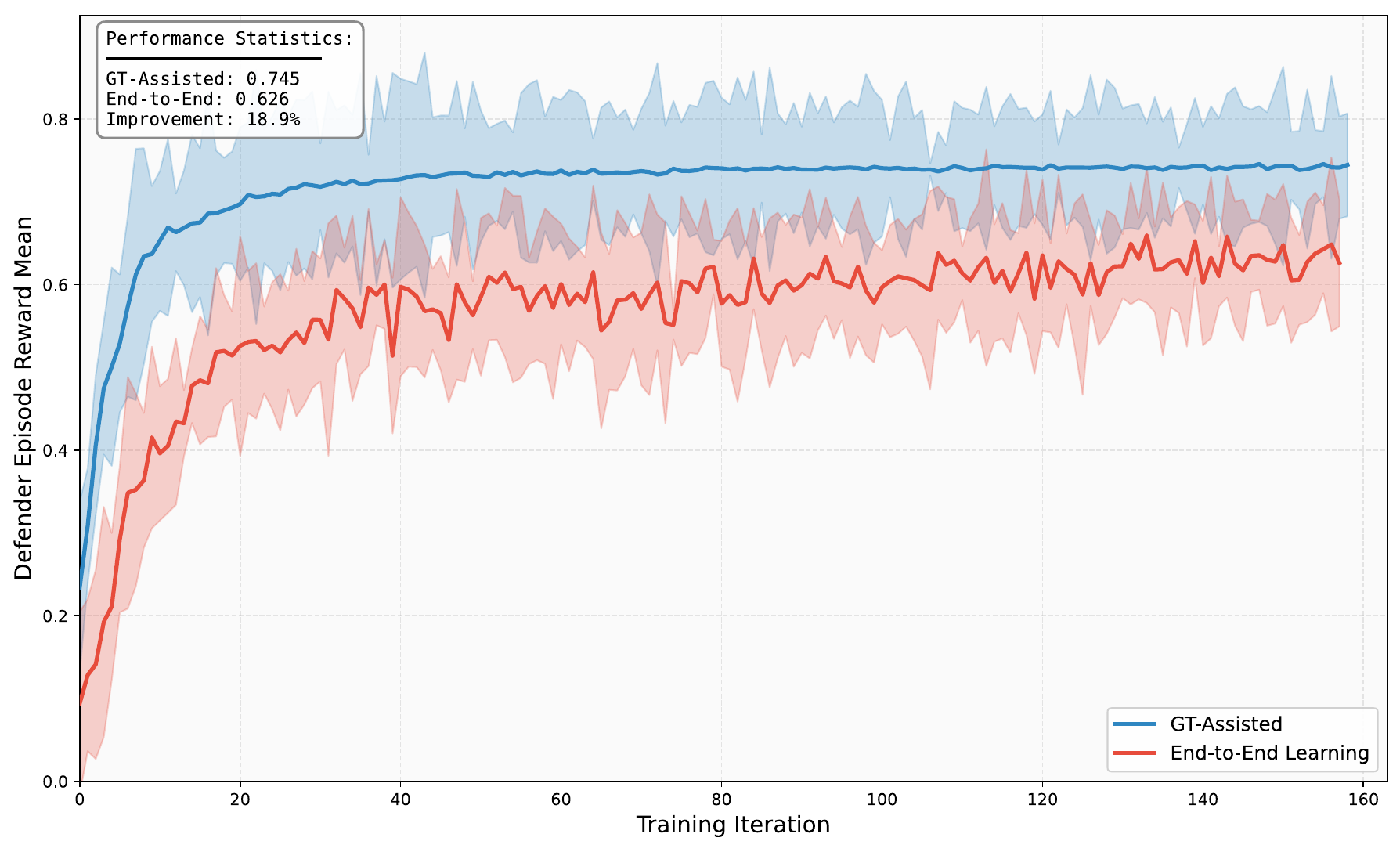}
    \caption{Training performance in the 3v1 homogeneous scenario. The GT-assisted approach (blue) again shows superior performance, achieving a stable mean reward of $\mu=0.745$ with lower variance. The end-to-end method (red) converges to a lower reward of $\mu=0.626$.}
    \label{fig:gt_vs_e2e_3v1}
\end{figure}

% \subsection{Heterogeneous Team Performance}
% \noindent
% We next evaluate a 2v1 heterogeneous scenario that requires explicit coordination among agents. The defender team is composed of a dedicated sensor agent ($\rho_s=0.35, \rho_c=0$) and a "blind" pursuer agent ($\rho_s=0, \rho_c=0.07$). 

% The experimental results demonstrate that the GT-assisted method is critical for learning this complex coordinated behavior. It achieves a stable converged reward approximately 10\% higher ( $\mu=0.407$ ) than the end-to-end baseline ( $\mu=0.367$ ). The standard learning approach struggles to bridge the credit assignment gap between the sensor's actions and the pursuer's success, leading to unstable training and significantly poorer performance.

\subsection{Spatial Configuration Quality Analysis}
\noindent
To evaluate whether GT-assisted training produces superior spatial configurations at sensing, we analyze the Nash equilibrium payoffs $J^*(\mathbf{x}(t_s))$ computed at sensing time for both approaches across 1,000 evaluation episodes in a homogeneous 3v1 environment with sensing radius $\rho_s = 0.15$ and speed ratio $\nu = 3.33$.

Figure~\ref{fig:spatial_eval} presents the comparison between the two scenarios. In particular, Figure~\ref{fig:spatial_eval}(c) reveals that GT-assisted training achieves higher Nash payoffs in successfully sensed episodes, with a median value of 0.768 versus 0.701 for end-to-end learning, which suggests better defender spatial configuration under GT-assisted training.

\section{Conclusion}
\noindent
This paper presents a hybrid framework that integrates game-theoretic analytical solutions with multi-agent reinforcement learning for a border defense game. By computing Nash equilibrium payoffs analytically at the moment of sensing, our approach eliminates the need to learn deterministic pursuit dynamics and concentrates computational resources on the challenging stochastic search problem. 

% Experimental evaluation across homogeneous and heterogeneous defender teams demonstrates that GT-assisted early termination achieves superior performance compared to standard end-to-end learning, with faster convergence rates, and better strategic spatial configurations at sensing time.
Experimental evaluation across homogeneous defender teams demonstrates that GT-assisted early termination achieves superior performance compared to standard end-to-end learning, with faster convergence rates and better strategic spatial configurations at sensing time.
Future work will extend this hybrid approach to scenarios with multiple simultaneous attackers, which would require handling partial sensing transitions as not all attackers may be detected at the same time, and investigate scalability to larger teams where the convex optimization for Apollonius Circle intersections may become more demanding. Incorporating additional game-theoretic insights into the search phase to further improve sample efficiency and coordination quality is also a promising direction.
\begin{figure}[tp]
    \centering
    \includegraphics[width=0.48\textwidth]{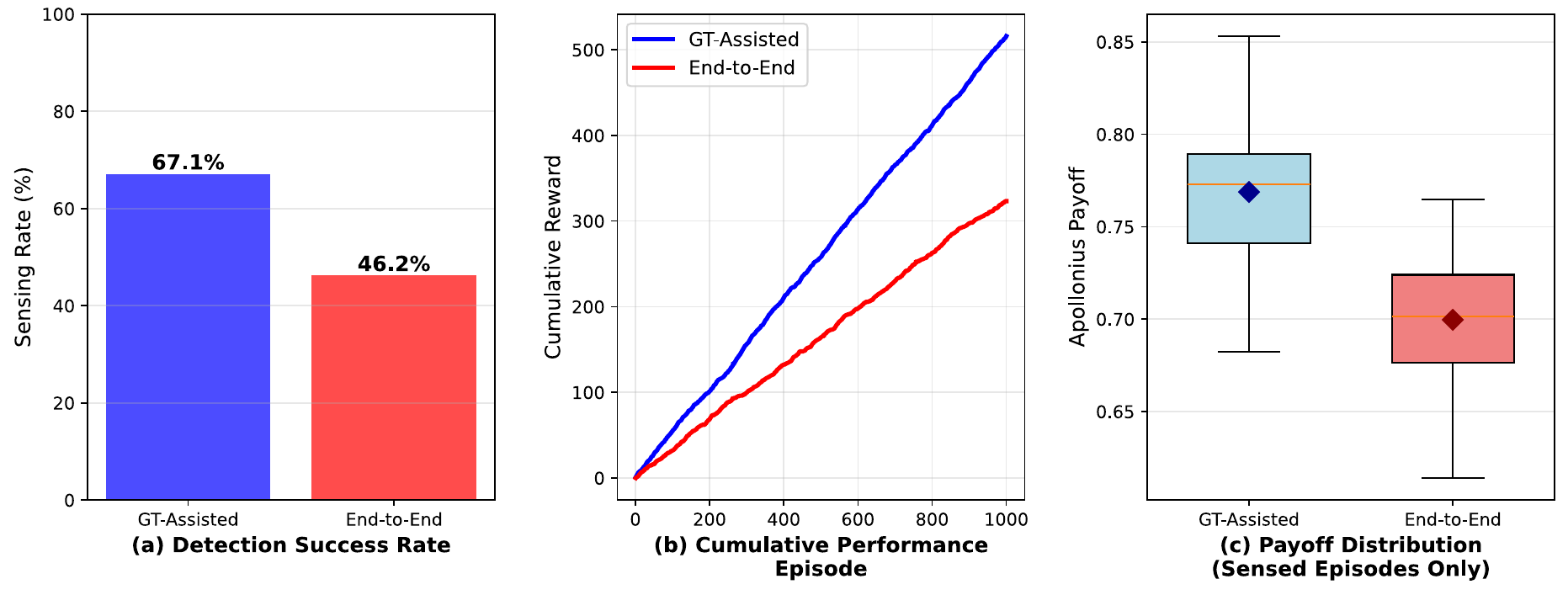}
    \caption{Evaluation of spatial configuration quality for the 3v1 homogeneous scenario ($\rho_s =0.15$) over 1,000 test episodes. (a) The GT-assisted policy achieves a much higher sensing rate (67.1\%) compared to the baseline (45.2\%). (b) This leads to superior cumulative performance. (c) For episodes where sensing occurred, the GT-assisted policy achieves a distribution of Apollonius payoffs with a significantly higher median ($0.768$ vs. $0.701$), indicating that it learns to position defenders more strategically at the moment of sensing.}
    \label{fig:spatial_eval}
\end{figure}
\typeout{} 
%===========================
\bibliographystyle{IEEEtran}
\bibliography{references}
\end{document}